\ifdefined\pdftexversion\pdfoutput=1\fi
\documentclass{article}
\def\DeFAbPublicDraft{1}
\def\DeFAbNamed{1}
\PassOptionsToPackage{hyphens}{url}
\usepackage[T1]{fontenc}
\usepackage{microtype}
\usepackage{graphicx}
\usepackage{booktabs,array,multirow,tabularx}
\usepackage{hyperref}

\ifdefined\DeFAbPublicDraft
\usepackage[preprint]{icml2026}
\hypersetup{pdfsubject={DeFAb research-direction preprint}}

  \def\DeFAbNamed{1}
\else
  \ifdefined\DeFAbCameraReady
    \usepackage[accepted]{icml2026}
    \def\DeFAbNamed{1}
  \else
    \usepackage{icml2026}
  \fi
\fi
\usepackage{amsmath,amssymb,amsthm}
\usepackage{textcomp}
\usepackage{xcolor}
\usepackage{tikz}
\usetikzlibrary{arrows.meta,positioning,calc,fit,backgrounds,shapes.misc}
\definecolor{dfNavy}{HTML}{1B2F4B}
\definecolor{dfBlue}{HTML}{3169A8}
\definecolor{dfTeal}{HTML}{00796F}
\definecolor{dfRed}{HTML}{B5473B}
\definecolor{dfGray}{HTML}{6B7280}
\definecolor{dfAmber}{HTML}{B7791F}
\definecolor{dfLight}{HTML}{EEF3F8}
\definecolor{dfTealLight}{HTML}{E6F2F1}
\definecolor{dfRedLight}{HTML}{F8ECEA}
\definecolor{dfAmberLight}{HTML}{FBF3E4}
\tikzset{
  dfbox/.style={draw=dfNavy!55, rounded corners=2pt, fill=dfLight, align=center,
    font=\footnotesize, inner sep=3pt, text=dfNavy},
  dfscope/.style={draw=dfTeal, dashed, rounded corners=4pt, fill=dfTealLight,
    inner sep=5pt},
  dfarrow/.style={-{Stealth[length=4pt]}, draw=dfGray, line width=0.6pt},
  dftitle/.style={font=\footnotesize\bfseries, text=dfNavy},
  dfnote/.style={font=\scriptsize, text=dfGray, align=center},
  dfaccept/.style={text=dfTeal},
  dfreject/.style={text=dfRed},
}

\newcommand{\D}{\mathcal D}
\newcommand{\E}{\mathcal E}
\newcommand{\K}{\mathcal K}

\icmltitlerunning{DeFAb: Grounded Scientific Discovery through Defeasible Abduction}

\begin{document}

\twocolumn[
\icmltitle{DeFAb: A Dataset for Grounded Scientific Discovery\\
  through Defeasible Abduction}
\ifdefined\DeFAbNamed
  % Named drafts only; the anonymous build never opens this file.
% Author order and Avyay Sadhu's CU Boulder affiliation confirmed by Patrick
% on October 22, 2026. The build checks order, affiliation, and PDF metadata.
\begin{icmlauthorlist}
\icmlauthor{Patrick Cooper}{cub}
\icmlauthor{Avyay Sadhu}{cub}
\icmlauthor{Alvaro Velasquez}{cub}
\end{icmlauthorlist}
\icmlaffiliation{cub}{University of Colorado Boulder, Boulder, Colorado, USA}
\icmlcorrespondingauthor{Patrick Cooper}{patrick.allen.cooper@gmail.com}
\hypersetup{pdfauthor={Patrick Cooper, Avyay Sadhu, and Alvaro Velasquez}}

\else
  \begin{icmlauthorlist}
    \icmlauthor{Anonymous Authors}{anon}
  \end{icmlauthorlist}
  \icmlaffiliation{anon}{Anonymous Institution}
  \icmlcorrespondingauthor{Anonymous Authors}{anon.email@domain.com}
\fi
\icmlkeywords{defeasible reasoning, abduction, scientific discovery, verification, dataset}
\vskip 0.3in
]
\printAffiliationsAndNotice{}

\begin{abstract}
A theory can accommodate a counterexample by adding an exception. The
harder problem is deciding how far the exception should extend and which
predictions should survive. We present DeFAb, a dataset and symbolic
verification framework built from curated biological knowledge in the Gene Ontology. Models
fill gaps in a theory or propose exceptions to its rules. A verifier checks
the consequences of each edit, allowing different solutions to receive
credit. Simple symbolic repairs solve all 336 tasks in the released
evaluation panel. This result exposes a limitation of the revision test.
An exception tied to one entity can succeed without identifying a pattern
that applies to other cases. We develop a contextual task that requires
an exception to reach the relevant cases and preserve predictions elsewhere.
Two planned studies compare training with verified and rejected examples,
and revision with and without feedback about consequences. DeFAb makes
the scope of a proposed explanation testable, providing a controlled step
toward studying theory revision in scientific reasoning.

\end{abstract}

% Main text. Section files are inputs; the order is the ICML outline of master
% plan C2. \label{maintextend} at the end marks the last main-text page for the
% page gate in build_icml.py (references, the Impact Statement and the appendix
% do not count toward the 8-page limit).
\section{Introduction}
\label{sec:intro}
% Conceptual illustration; animals are a toy theory, not a GO record.
\begin{figure*}[t]
\centering
\begin{tikzpicture}[x=1cm,y=1cm,font=\fontsize{9}{11}\selectfont]
\path[use as bounding box] (0,0) rectangle (16.7,6.65);
\tikzset{
 penguin/.pic={
  \fill[dfNavy] (0,.46) ellipse (.24 and .43);
  \fill[dfNavy] (0,.85) circle (.21);
  \fill[white] (0,.40) ellipse (.16 and .31);
  \fill[white] (.07,.87) circle (.07);
  \fill[dfNavy] (.09,.89) circle (.025);
  \fill[dfAmber] (.17,.81) -- (.35,.77) -- (.17,.72) -- cycle;
  \fill[dfNavy] (-.17,.62) .. controls (-.39,.45) and (-.37,.24) .. (-.28,.20) -- (-.12,.47) -- cycle;
  \fill[dfAmber] (-.05,.06) -- (-.22,-.02) -- (.01,-.02) -- cycle;
  \fill[dfAmber] (.10,.06) -- (.02,-.02) -- (.26,-.02) -- cycle;
 },
 songbird/.pic={
  \fill[dfTeal] (0,.45) ellipse (.32 and .22);
  \fill[dfTeal] (.23,.64) circle (.16);
  \fill[dfTeal] (-.23,.48) -- (-.49,.66) -- (-.37,.37) -- cycle;
  \fill[dfNavy!45] (-.18,.51) .. controls (.01,.68) and (.19,.51) .. (.05,.32) .. controls (-.10,.33) .. (-.18,.51);
  \fill[white] (.28,.68) circle (.035);
  \fill[dfAmber] (.37,.64) -- (.54,.59) -- (.36,.55) -- cycle;
  \draw[dfNavy,line width=.7pt] (-.06,.26) -- (-.06,.06) -- (.05,.06);
  \draw[dfNavy,line width=.7pt] (.09,.26) -- (.09,.06) -- (.20,.06);
 },
 wing/.pic={\draw[line width=.9pt] (-.13,0) .. controls (-.07,.13) and (0,.10) .. (0,0) .. controls (.07,.13) and (.13,.10) .. (.15,0);}
}
\node[anchor=west,inner sep=0pt,font=\bfseries\fontsize{11}{13}\selectfont,text=dfNavy] at (.22,6.32)
 {One observation, three possible revisions};
\node[anchor=west,inner sep=0pt,text=dfGray] at (.22,5.85)
 {A theory predicts that birds fly. An observed penguin challenges that prediction.};
% A visual key explains prediction status without arrows pointing at the birds.
\fill[dfNavy] (.46,5.25) circle (.20);
\begin{scope}[shift={(.46,5.23)},draw=white]\pic {wing};\end{scope}
\node[anchor=west,text=dfGray] at (.80,5.25) {Flight still predicted};
\draw[dfGray!55,line width=.7pt,fill=white] (5.35,5.25) circle (.20);
\begin{scope}[shift={(5.35,5.23)},draw=dfGray!50]\pic {wing};\end{scope}
\draw[dfGray,line width=.9pt] (5.20,5.10) -- (5.50,5.40);
\node[anchor=west,text=dfGray] at (5.69,5.25) {Flight prediction withdrawn};
\draw[dfGray!25] (5.5,.2) -- (5.5,4.8);
\draw[dfGray!25] (11.2,.2) -- (11.2,4.8);
\node[text=dfNavy,font=\bfseries\fontsize{10}{12}\selectfont] at (2.65,4.65) {Except this penguin};
\node[text=dfTeal,font=\bfseries\fontsize{10}{12}\selectfont] at (8.35,4.65) {Except penguins};
\node[text=dfNavy,font=\bfseries\fontsize{10}{12}\selectfont] at (14.05,4.65) {Except all birds};
% The tinted region is the exception's scope, not an algorithmic box.
\draw[draw=dfBlue!70,dashed,fill=dfBlue!7,rounded corners=18pt] (.22,2.02) rectangle (1.48,4.12);
\draw[draw=dfTeal,dashed,fill=dfTeal!8,rounded corners=18pt] (5.92,2.02) rectangle (8.35,4.12);
\draw[draw=dfRed!75,dashed,fill=dfRed!6,rounded corners=18pt] (11.62,2.02) rectangle (16.53,4.12);
\foreach \s in {0,5.7,11.4} {
 \begin{scope}[shift={(\s,0)}]
  \pic at (.85,2.28) {penguin};
  \pic at (2.02,2.28) {penguin};
  \pic at (3.25,2.28) {songbird};
  \pic at (4.45,2.28) {songbird};
  \node[align=center,text=dfGray,font=\fontsize{8}{9.5}\selectfont] at (.85,1.55) {Observed\\penguin};
  \node[align=center,text=dfGray,font=\fontsize{8}{9.5}\selectfont] at (2.02,1.55) {Another\\penguin};
  \node[align=center,text=dfGray,font=\fontsize{8}{9.5}\selectfont] at (3.85,1.55) {Other birds};
 \end{scope}
}
% Flight prediction retained: a filled disk; withdrawn: an open crossed disk.
\foreach \x in {2.02,3.25,4.45,8.95,10.15} {
 \fill[dfNavy] (\x,3.70) circle (.24);
 \begin{scope}[shift={(\x,3.68)},draw=white]\pic {wing};\end{scope}
}
\foreach \x in {.85,6.55,7.72,12.25,13.42,14.65,15.85} {
 \draw[dfGray!55,line width=.7pt,fill=white] (\x,3.70) circle (.24);
 \begin{scope}[shift={(\x,3.68)},draw=dfGray!50]\pic {wing};\end{scope}
 \draw[dfGray,line width=.9pt] (\x-.18,3.52) -- (\x+.18,3.88);
}
% Visually identify the missed transfer and collateral damage.
\draw[dfAmber,line width=1.3pt] (2.02,3.70) circle (.30);
\foreach \x in {14.65,15.85} {\draw[dfRed,line width=1.3pt] (\x,3.70) circle (.30);}
\node[align=center,text=dfAmber,font=\bfseries] at (2.65,.65)
 {Misses the other penguin};
\node[align=center,text=dfTeal,font=\bfseries] at (8.35,.65)
 {Changes exactly the intended cases};
\node[align=center,text=dfRed,font=\bfseries] at (14.05,.65)
 {Loses valid predictions elsewhere};
\end{tikzpicture}
\caption{Choosing the scope of an exception. The toy theory initially predicts
flight for all four birds. Shading marks the scope of a proposed exception;
filled flight symbols are retained predictions and crossed symbols are
withdrawn predictions. Withdrawal does not assert flightlessness.
The toy task stipulates the intended scope; one observation would not
establish it empirically. A local repair resolves the observed conflict but
misses transfer (amber ring). A universal repair also removes predictions
that should remain (red rings). The middle revision satisfies both obligations.
This constructed illustration explains the proposed contextual contract; it is
not a GO record or a model result. The released L3 tasks check resolution and
preservation without a designated transfer set.}
\label{fig:hero}
\end{figure*}

When an observation challenges a theory, finding an exception is only part
of the problem. An exception confined to the observed case can accommodate
the evidence without explaining anything beyond it. A broader exception
may capture a useful pattern, but it can also overturn predictions that
should remain intact. A satisfactory revision must identify which other
cases should change with the observation and which parts of the theory
should survive.

Figure~\ref{fig:hero} illustrates the distinction. A toy theory predicts
that birds fly. The constructed task identifies penguins as the exception,
with one penguin supplying the observed counterexample. Exempting that
penguin leaves the same mistake for another penguin. Exempting every bird
loses predictions that should stand. The useful revision identifies the
class for which the default fails. Its value lies in the pattern of
consequences it changes, not simply in resolving the initial conflict.

This problem arises whenever reasoning depends on rules that normally
hold but admit exceptions. Such reasoning is called \emph{defeasible}
because a conclusion can lose its support when an exception applies.
\emph{Abduction} proposes an assumption that could account for an
observation or restore a missing conclusion. Together, these ideas let
us study how models construct and revise explanations in a setting where
the effects of their proposals can be checked.

DeFAb turns this problem into tasks built from finite theories grounded
in curated biological knowledge from the Gene Ontology. Models supply a missing fact, supply
a missing rule, or introduce an exception to a challenged prediction.
A symbolic verifier computes what follows after each proposed edit. In
revision tasks, the edit must withdraw the challenged prediction and
preserve the theory's other predictions. Different edits can satisfy
these conditions. Success therefore depends on their consequences rather
than agreement with a single reference answer. Source records connect
each task to the knowledge used to construct it.

The released evaluation exposes a limitation of those requirements.
Simple symbolic repairs solve all 336 tasks in the held-out panel, and
two of six evaluated model deployments succeed on every task under both
prompt styles. In the revision tasks, an exception tied to the observed
entity can meet the contract without identifying a class of cases that
should be treated alike. Formal correctness is reproducible, but does
not by itself establish that a model has chosen a useful scope.

We develop a contextual contract to make that scope testable. Alongside
the challenged case, a task includes entities whose predictions should
also change and entities whose predictions must remain. The verifier
checks both obligations and accepts every permitted edit that satisfies
them. A complete symbolic solution procedure provides a standard against
which learned revisions can be assessed. The question is whether models
can acquire this ability and carry it to new families of problems.

Two prospective studies distinguish the possible contributions of
verification. One compares training on verified revision examples with
matched examples rejected by the verifier. The other compares revision
with and without feedback about a proposal's consequences, under matched
call and output limits. The contextual corpus and these learning studies
are still under development. The present contribution is the released
instrument and its evaluation, a formal account of the scope it leaves
untested, and a controlled design for studying that scope. These constructed
problems isolate one part of scientific reasoning. An empirical discovery
would additionally require evidence beyond the supplied theory.

\section{Related work}
\label{sec:related}
% Citation-gated additions (plan C2) landed on 2026-10-07; each key below was
% verified against a Crossref or arXiv record in CITATION_LOG.md before it
% entered the three bibliography files. TransE (Bordes et al. 2013) has no
% Crossref or arXiv record and is therefore not cited (see the log).

DeFAb connects the proposal of changes to a theory with independent
checking of their consequences. Abductive
logic programming and defeasible logic provide foundations for
explanations and exceptions
\citep{kakas1992abductive,antoniou2001defeasible}. FORTE studies
revision of Horn-clause theories against examples
\citep{richards1995forte,wrobel1994concept}, while belief revision
motivates preserving commitments when accommodating new information
\citep{alchourron1985logic,dalal1988investigations}. DeFAb turns
preservation into an explicit acceptance condition and, in the
proposed contextual extension, also specifies where an exception
must transfer.

Existing language-model benchmarks examine several parts of this
problem. $\alpha$NLI, ProofWriter, and AbductionRules study explanation
selection, verified single-fact abduction, and explanation generation
\citep{bhagavatula2020abductive,tafjord2021proofwriter,young2022abductionrules}.
$\delta$-NLI, Belief-R, BoardgameQA, and DeFReAsING examine defeasible
updates, retraction, source preferences, and exceptions
\citep{rudinger2020skeptic,wilie2024belief,kazemi2023boardgameqa,allaway2025defreasing}.
LogicBench, FOLIO, and PrOntoQA provide related evaluations of
logical inference
\citep{parmar2024logicbench,han2024folio,saparov2023prontoqa}.
Knowledge editing likewise evaluates changed commitments and
side effects \citep{meng2022locating,cohen2024evaluating}.
Our focus is the consequence set of an explicit theory edit,
including alternative edits that satisfy the same obligations.

Independent checking also connects DeFAb to systems that separate
proposal generation from evaluation. LCF and Lean embody this
separation in formal reasoning
\citep{gordon1979lcf,demoura2021lean4}; FunSearch, AlphaEvolve,
and self-improving agents use supplied evaluators to assess
generated candidates
\citep{romeraparedes2024funsearch,novikov2025alphaevolve,zhang2025dgm}.
In DeFAb, the evaluator checks a finite theory after an edit,
making both successful revision and collateral loss inspectable.

Scientific-discovery benchmarks address broader activities,
including data-driven hypotheses, scientific programming,
experimental cycles, and recovery of hypotheses
\citep{majumder2024discoverybench,chen2024scienceagentbench,
jansen2024discoveryworld,wang2024scimon,yang2024moosechem,
liu2025hypobench}. Certified tasks can also constrain which
completions remain possible \citep{shi2026jump}. DeFAb isolates
the narrower question of how an explanation changes a supplied
theory. Its GO grounding connects it to scientific knowledge
without making it a protein-function prediction task such as
CAFA or DeepGO
\citep{radivojac2013cafa,kulmanov2018deepgo}. GO taxon constraints
provide traceable exceptions for the contextual construction
\citep{deegan2010taxon}.

\section{A verifiable revision problem}
\label{sec:method}
\subsection{Theories and three task levels}
\label{sec:levels}
A DeFAb answer changes an explicit theory. We judge the answer by
recomputing the theory's predictions and checking whether the intended
change occurred. This makes the consequences of an edit available for
inspection, including consequences that the proposer did not mention.

We represent a theory as $\D=(F,R_s,R_d,R_{df},\succ)$, containing facts,
strict rules ($\rightarrow$), defaults ($\Rightarrow$), defeaters
($\leadsto$), and an acyclic priority relation. Defaults support
conclusions that exceptions may block. A defeater attacks support for
its head's complement without establishing its own head. We use the
ambiguity-blocking, team-defeat proof theory of
\citet{antoniou2001defeasible}, writing $\D\vdash+\Delta q$ for strict
provability and $\D\vdash+\partial q$ for defeasible provability.
Explicit negation $\bar q$ is distinct from failure to derive $q$.
After grounding over the record's finite set of constants, the theory's
expectations are
\[
  \E(\D)=\{p:\D\vdash+\partial p\}.
\]
An element is critical for a target $q$ when removing it makes a
previously derivable $q$ underivable. Each task records the resulting
challenge theory $\D^-$, the target, the permitted edits, the source
provenance, and verified witnesses.

\textbf{L1, fact completion,} removes a critical ground fact.
\textbf{L2, rule abduction,} removes a critical rule. In both cases,
$q\notin\E(\D^-)$ and an admissible edit must restore $q$.
\textbf{L3, conservative revision,} removes an exception and its
priorities, leaving $q\in\E(\D^-)$ but $\D^-\not\vdash+\Delta q$.
A constructed observation $\alpha=\bar q$ challenges that prediction.
The observation is not inserted as a fact: otherwise it could directly
settle the very revision problem being evaluated.

An answer supplies an admissible edit $h$ (a fact at L1, a rule at L2/L3)
and, where required, priority edges $\Gamma$, giving $\D_h=\D^-\oplus(h,\Gamma)$. In L3, resolving the
anomaly means withdrawing the contradicted expectation:
$q\notin\E(\D_h)$. Deriving $\bar q$ is not required; blocking the default
is sufficient. Under the first contract, every other expectation must
remain:
\begin{equation}\label{eq:accept1}
\begin{split}
\operatorname{accept}_1(h,\Gamma)={}&\operatorname{admissible}(h,\Gamma)\\
&\land[q\notin\E(\D_h)]\land[\K\subseteq\E(\D_h)],
\end{split}
\end{equation}
where $\K=\E(\D^-)\setminus\{q\}$. Resolution and preservation are also
recorded separately; joint correctness requires both.

\subsection{Contextual revision and transfer}
\label{sec:acceptance}
A repair that names only the observed entity may satisfy the released
contract. It leaves a further question unanswered: when should the
same exception apply elsewhere? Figure~\ref{fig:hero} illustrates why
this matters. A repair can be too narrow to cover another relevant
case, or too broad to preserve predictions that remain justified.

The proposed \texttt{defab/2} contract makes these consequences part of
the task. Alongside the observed entity, it supplies a finite context
containing cases that should share the exception and cases that should
not. The transfer set $\mathit{Tr}$ contains additional predictions
that must be withdrawn. The retention set $\mathit{Kr}$ identifies
predictions outside the exception's scope that must remain.

Formally, require
$\mathit{Tr}\subseteq\E(\D^-)\setminus\{q\}$ and define
$W=\{q\}\cup\mathit{Tr}$ and $\K_2=\E(\D^-)\setminus W$, with
$\mathit{Kr}\subseteq\K_2$. An acceptable revision withdraws every
prediction in $W$ and preserves every prediction in $\K_2$:

\begin{equation}\label{eq:accept2}
\begin{split}
\operatorname{accept}_2(h,\Gamma)={}&\operatorname{admissible}(h,\Gamma)\\
&\land[W\cap\E(\D_h)=\emptyset]\\
&\land[\K_2\subseteq\E(\D_h)].
\end{split}
\end{equation}
The context entities and their public evidence are visible in the prompt;
their transfer and retention labels are private. Retention is a designated
part of the complete preservation requirement. Reporting the fraction of
each set handled correctly helps diagnose errors, but partial satisfaction
does not replace joint success.

For example, a default specializes a parent-term annotation to a child
term $X$. A taxon exception should prevent that specialization for entities
of the relevant taxon, yet preserve it for allowed taxa. Conditioning on
$e_1$ alone resolves the public anomaly but misses transfer. Conditioning
on every entity withdraws too much. A class predicate that separates the
transfer and retention entities meets the contract. More than one predicate
may do so; all admissible successful alternatives must be accepted.

\paragraph{Finite proposal space.} Each prompt specifies permitted
symbols, rule form, variable bound, and priorities. The contextual language
allows one exception rule with a task-fixed head and a bounded class,
membership, or taxon condition. It permits one body literal and at most one
variable; the task fixes the head and allows a defeater or a defeasible
exception with the required priority. Bodies must apply at the observed entity;
new symbols and circular or malformed edits are rejected. Exhaustive
bounded enumeration establishes each record's valid-answer set. The
verifier, rather than the reference, determines acceptance. This makes
alternative solutions and chance baselines measurable while keeping the
task a scoped formal problem, not unrestricted hypothesis generation.

\subsection{What verification establishes}
\label{sec:formal}
The evaluator computes proof-tag closure over the grounded theory, applies
the edit, and compares consequence sets. Efficient propositional defeasible
reasoning provides the underlying computational basis
\citep{maher2001propositional}. Admissibility, resolution, and preservation
are checked separately, so parser failures cannot be reported as
semantically unsuccessful admissible hypotheses.

Equation~\ref{eq:accept2} has a useful immediate interpretation:
for an admissible edit under the L3 preconditions, acceptance is equivalent
to
\begin{equation}\label{eq:removed}
\E(\D^-)\setminus\E(\D_h)=W.
\end{equation}
The withdrawal clause requires every element of $W$ to disappear,
and preservation forbids any other disappearance. New consequences do not
invalidate this identity. Taking $\mathit{Tr}=\emptyset$ recovers
Equation~\ref{eq:accept1}. This is a set-theoretic property of the contract,
not an empirical generalization claim.

A verified answer is correct relative to the supplied finite theory and
edit language. Provenance makes those assumptions inspectable. Establishing
that the edit describes a biological mechanism would require independent
observation or experiment.

\section{From curated knowledge to revision tasks}
\label{sec:dataset}
\subsection{The released dataset}
The public release uses GO annotations and ontology structure
\citep{ashburner2000go,goc2026knowledgebase}. Source assertions are projected
into bounded gene-product neighborhoods, translated to facts and rules,
and used to construct critical-element tasks. A record traces its source
family, projection, challenge, and checked witnesses. Provenance enables
inspection and grouping; it does not imply statistical independence of
all records derived from a neighborhood.

Release 1.0.0 contains 225,507 records: 112,692 L1, 112,692 L2, and 123 L3.
They span 2,099 source families. The revision population is therefore much
smaller than the aggregate size suggests, a distinction that matters when
interpreting the evaluation. Appendix~\ref{app:released} gives the fixed
partition sizes.

The contextual tasks use a different acceptance contract and will be
evaluated separately. Keeping these populations distinct makes clear
which claims follow from the released evidence and which await the new
instrument.

\subsection{Contextual construction under development}
\label{sec:construction2}
The contextual design uses GO taxon constraints to construct
exceptions with a meaningful scope \citep{deegan2010taxon}. A broad
annotation does not by itself establish a more specific one. We
deliberately make that specialization a defeasible default, supported
by an annotation pattern among allowed entities. A taxon constraint
identifies cases in which the specialization should be blocked.
Withholding the exception creates a controlled challenge to the
theory's prediction. The curated constraint and the constructed
default retain separate provenance.

Each context includes the observed entity, additional entities within
the exception's intended scope, and allowed entities outside it.
A contextual family groups a child term, constraint type, taxon,
parent term, and annotation relation. Rotating entities and varying
the allowed contexts tests whether a proposed condition captures the
relevant class rather than an incidental identifier.

Selection requires support of at least $0.25$, measured as the
fraction of eligible allowed-side entities annotated to the child
rather than only its parent. A stricter sub-panel requires support
of at least $0.5$ at each construction stage. The support calculation,
context bounds, and prompt limits are specified in
Appendix~\ref{app:agenda}.

\paragraph{Two generalization conditions.} One training population supports
separate family-disjoint and source-taxon-disjoint evaluations. The former
holds out source families while representing their taxon-constraint
classes in training; the latter reserves 20\% of source classes by a fixed
hash ordering. A source class groups taxon constraints with the same
membership footprint over the source species. Shared constraints, families, and context entities remain
grouped. An overlap audit considers \emph{all valid answers}, not only the
reference: a new source class may still admit an answer class already seen
in training. Source-class novelty alone cannot justify unseen-answer
transfer.

The family-disjoint panel and the smaller taxon challenge have separate
yield targets, given in Appendix~\ref{app:agenda}. These targets have not
yet been achieved. A shortfall limits the corresponding generalization
claim; it does not justify choosing a new partition after seeing model
performance. Any split amendment must precede measurement, with the
original population retained.

\paragraph{Instrument admission.} A frozen population must pass its
ten binding instrument checks before model experiments. These cover task
semantics, accepted alternatives, ground-template and circular-edit
controls, over-broad withdrawals, prompt and candidate-order leakage,
grouping and entity/constraint overlap, support, attrition, and reconciled
counts. Baselines include two simple taxon heuristics, a symbolic search
that infers the required scope from public evidence, train-only retrieval,
random admissible edits, and an oracle. The symbolic search must solve
every admitted contextual task. This is a prospective validity requirement,
not a measured result on the unfinished corpus. The learning studies ask
whether models acquire and generalize this available procedure; they do
not seek to outperform a complete symbolic solver. All heuristic successes
remain in the population.

\section{What the released evaluation establishes}
\label{sec:results}
\subsection{Panel and outcome accounting}
The held-out 1.0.0 panel contains 336 tasks: 150 L1, 150 L2, and 36 L3,
including 29 L3 families. No L3 family contributes
more than three tasks. Six pinned hosted deployments were evaluated with
two prompt styles, for 672 responses per deployment and 4,032 in all.
Both styles used the same tasks and acceptance contract, with a
4,096-token output ceiling. These are deployment-specific reference
results, not a separately verified contemporary SOTA cohort.

Figure~\ref{fig:outcomes} shows success under the two prompt styles.
Table~\ref{tab:released} in Appendix~\ref{app:released} gives exact counts
and partitions failures by type.
A response-contract violation includes an output that cannot supply the
required admissible answer. A capped output is recorded separately rather
than interpreted as an ordinary semantic error. A semantic miss is a
parsed, admissible hypothesis that does not resolve the target. This
decomposition matters: interface failure is not evidence that an
admissible revision damages the theory.

% Generated by evidence.py from registered, replayed evidence; do not edit.
\begin{figure}[t]\centering
\begin{tikzpicture}[x=1cm,y=1cm,font=\fontsize{8}{9.5}\selectfont]
\path[use as bounding box] (0,0) rectangle (8.05,5.55);
\node[anchor=west,font=\bfseries,text=dfNavy] at (0,5.26) {Joint success on the released panel};
\fill[dfNavy] (2.8,4.85) rectangle (3.05,4.96);
\node[anchor=west] at (3.12,4.9) {Direct};
\fill[dfTeal] (4.65,4.85) rectangle (4.9,4.96);
\node[anchor=west] at (4.97,4.9) {Structured};
\draw[dfGray!20] (2.800000,.65) -- (2.800000,4.57);
\node[text=dfGray] at (2.800000,.36) {0\%};
\draw[dfGray!20] (3.937500,.65) -- (3.937500,4.57);
\node[text=dfGray] at (3.937500,.36) {25\%};
\draw[dfGray!20] (5.075000,.65) -- (5.075000,4.57);
\node[text=dfGray] at (5.075000,.36) {50\%};
\draw[dfGray!20] (6.212500,.65) -- (6.212500,4.57);
\node[text=dfGray] at (6.212500,.36) {75\%};
\draw[dfGray!20] (7.350000,.65) -- (7.350000,4.57);
\node[text=dfGray] at (7.350000,.36) {100\%};
\node[anchor=east,text=dfNavy] at (2.60,4.280000) {GPT-5.6 Sol};
\fill[dfNavy] (2.8,4.330000) rectangle (7.350000,4.450000);
\fill[dfTeal] (2.8,4.120000) rectangle (7.350000,4.240000);
\node[anchor=east,text=dfNavy] at (2.60,3.620000) {Claude Sonnet 5};
\fill[dfNavy] (2.8,3.670000) rectangle (6.943750,3.790000);
\fill[dfTeal] (2.8,3.460000) rectangle (6.104167,3.580000);
\node[anchor=east,text=dfNavy] at (2.60,2.960000) {GPT-6.1 Sol};
\fill[dfNavy] (2.8,3.010000) rectangle (7.350000,3.130000);
\fill[dfTeal] (2.8,2.800000) rectangle (7.350000,2.920000);
\node[anchor=east,text=dfNavy] at (2.60,2.300000) {Claude Sonnet 5.5};
\fill[dfNavy] (2.8,2.350000) rectangle (4.601042,2.470000);
\fill[dfTeal] (2.8,2.140000) rectangle (7.350000,2.260000);
\node[anchor=east,text=dfNavy] at (2.60,1.640000) {DeepSeek-V4-Pro};
\fill[dfNavy] (2.8,1.690000) rectangle (6.496875,1.810000);
\fill[dfTeal] (2.8,1.480000) rectangle (3.314583,1.600000);
\node[anchor=east,text=dfNavy] at (2.60,0.980000) {Kimi-K2.6};
\fill[dfNavy] (2.8,1.030000) rectangle (5.576042,1.150000);
\fill[dfTeal] (2.8,0.820000) rectangle (4.140625,0.940000);
\end{tikzpicture}
\caption{The same tasks can yield very different success rates under two prompt styles. Each bar contains 336 tasks. Two pinned deployments succeed throughout both styles. These are descriptive reference runs under deployment-specific settings, not a controlled cross-provider ranking. Exact counts and failure types appear in Table~\ref{tab:released}.}
\label{fig:outcomes}\end{figure}

Two deployments succeed on all 672 task--style combinations. The
other four produce 1,054 failures, comprising 652 response-contract
violations, 366 outputs that reach the token limit, and 36 admissible
answers that miss the semantic target. Most observed failures
therefore occur before an admissible answer can be assessed
semantically.

Among evaluated admissible L3 edits, no collateral loss was recorded.
This observation concerns the revision tasks only; L1/L2 tasks and
invalid answers do not provide preservation trials. Together with
the pronounced prompt-style effects, these results give little
evidence of a model failure to preserve the supplied theory.
Understanding the panel's difficulty requires examining the repairs
that its acceptance rules permit.

\subsection{A shortcut and a useful negative result}
Generic repair and bounded symbolic search solve all 336 panel tasks
under both prompt styles. A census of the release helps explain why:
all 3,144,222 rule bodies consist of a single entity fact. This
regularity permits an entity-local repair without requiring the
solver to identify a transferable exception class. The result
exposes a limitation of the test. Success on the observed case
does not reveal how a solver would choose the exception's scope.

The verifier can nevertheless detect damage outside that case.
A separate checker probe applies the generalized circular defeater
$R(e,X)\leadsto\neg R(e,X)$. It resolves the anomaly on all 36 L3
tasks but removes protected predictions on 34, spanning all
29 L3 families. The verifier rejects those 34 edits. This probe
tests the checker with deliberately constructed answers; it does
not describe behavior observed in the evaluated models. Such
circular bodies are excluded by the proposed contextual language.

These findings motivate the contextual contract. The released
verifier can distinguish a local repair from an edit that damages
protected predictions, but the released task does not require
transfer to additional cases. Adding that obligation makes the
scope of the exception observable. The following studies ask
whether verified examples and consequence feedback help models
meet this stronger contract.

\section{Testing the value of verification}
\label{sec:protocol}
% Two separate experiments; every arrow is a single left-to-right path.
\begin{figure*}[t]
\centering
\begin{tikzpicture}[x=1cm,y=1cm,font=\fontsize{9}{11}\selectfont]
\path[use as bounding box] (0,0) rectangle (16.7,7.4);
\tikzset{
 model/.pic={
  \foreach \a in {-.22,.22} \foreach \b in {-.32,0,.32}
   {\draw[dfNavy!35,line width=.5pt] (-.28,\a) -- (.08,\b); \draw[dfNavy!35,line width=.5pt] (.08,\b) -- (.42,\a);}
  \foreach \y in {-.22,.22} {\fill[dfNavy] (-.28,\y) circle (.075);\fill[dfNavy] (.42,\y) circle (.075);}
  \foreach \y in {-.32,0,.32} {\fill[dfNavy] (.08,\y) circle (.075);}
 },
 revise/.pic={
  \pic {model};
  \draw[-{Stealth[length=4.5pt,width=3.6pt]},draw=dfBlue,line width=.85pt]
   (.08,.60) arc[start angle=90,end angle=385,radius=.60];
 },
 bflow/.style={-{Stealth[length=5.5pt,width=4pt]},draw=dfNavy!75,line width=.8pt},
 testset/.pic={
  \draw[fill=white,draw=dfNavy!40,rounded corners=1pt] (-.28,-.32) rectangle (.28,.32);
  \foreach \y in {-.15,.05,.23} {\draw[dfNavy!55] (-.17,\y) circle (.035);\draw[dfNavy!35] (-.06,\y) -- (.17,\y);}
 },
 data/.pic={
  \foreach \y in {-.20,0,.20} {\draw[draw=dfNavy!35,fill=dfLight] (-.40,\y-.10) rectangle (.40,\y+.13);}
 }
}
\node[anchor=west,font=\bfseries\fontsize{11}{13}\selectfont,text=dfBlue] at (0,7.05) {A\quad Does verified training help?};
\node[anchor=west,font=\bfseries\fontsize{11}{13}\selectfont,text=dfTeal] at (8.0,7.05) {B\quad Does verifier feedback help?};
\node[anchor=west,text=dfGray,font=\fontsize{8}{9.5}\selectfont] at (0,6.58) {Same model and training limits; three seeds};
\node[anchor=west,text=dfGray,font=\fontsize{8}{9.5}\selectfont] at (8.0,6.58) {Each arm on the base model and three verified-SFT models};
\draw[dfGray!35,double,double distance=2pt] (7.45,.45) -- (7.45,7.3);
\foreach \x/\s in {1/Data,3.35/Train,5.65/Test,8.55/First call,12.45/Revise,14.65/Test} {
 \node[text=dfGray,font=\fontsize{8}{9.5}\selectfont] at (\x,6.0) {\s};
}
% Study A: three parallel, matched paths. Colored top sheet is the L3 completion.
\foreach \y in {5.2,3.45,1.7} {
 \pic at (1,\y) {data};\pic at (3.35,\y) {model};\pic at (5.65,\y) {testset};
 \draw[dfarrow] (1.55,\y) -- (2.80,\y);
 \draw[dfarrow] (3.95,\y) -- (5.13,\y);
}
\draw[fill=dfTeal!18,draw=dfTeal] (.6,3.55) rectangle (1.4,3.78);
\draw[dfTeal,line width=1pt] (.90,3.65) -- (.98,3.59) -- (1.12,3.75);
\draw[fill=dfRed!13,draw=dfRed] (.6,1.80) rectangle (1.4,2.03);
\draw[dfRed,line width=1pt] (.94,1.84) -- (1.09,1.99) (.94,1.99) -- (1.09,1.84);
\node[text=dfGray] at (1,4.53) {L1/L2 only};
\node[text=dfTeal] at (1,2.78) {L1/L2 + verified L3};
\node[text=dfRed] at (1,1.03) {L1/L2 + rejected L3};
\node[text=dfGray,font=\fontsize{8}{9.5}\selectfont] at (5.65,4.53) {reference};
\draw[dfBlue,line width=1.1pt] (6.12,3.45) -- (6.50,3.45) -- (6.50,1.70) -- (6.12,1.70);
\node[fill=white,text=dfBlue,align=center,font=\bfseries\fontsize{8}{9}\selectfont] at (6.5,2.575) {Primary\\contrast};
% Study B: no shared branching, crossed arrows, or alignment with Study A cells.
\foreach \y in {5.2,3.45,1.7} {\pic at (8.55,\y) {model};\pic at (14.65,\y) {testset};}
\foreach \y in {3.45,1.7} {
 \pic at (12.45,\y) {revise};
 \draw[bflow] (13.25,\y) -- (14.22,\y);
}
\draw[bflow] (9.18,5.2) -- (14.22,5.2);
\node[fill=white,inner xsep=4pt,text=dfGray,font=\itshape\fontsize{8}{9.5}\selectfont]
 at (12.45,5.2) {no revision};
\draw[bflow] (9.18,3.45) -- (11.80,3.45);
\draw[bflow] (9.18,1.7) -- (10.08,1.7);
\draw[bflow] (11.02,1.7) -- (11.80,1.7);
\draw[dfTeal,fill=dfTealLight,line width=.8pt] (10.55,1.7) circle (.32);
\draw[dfTeal,line width=.7pt,fill=white] (10.36,1.46) rectangle (10.74,1.94);
\draw[dfGray,line width=.7pt] (10.40,1.81) -- (10.70,1.81);
\draw[dfRed,line width=.9pt] (10.47,1.87) -- (10.60,1.75);
\draw[dfTeal,line width=1.1pt] (10.40,1.59) -- (10.70,1.59);
\node[text=dfGray] at (10.10,5.52) {Direct};
\node[text=dfGray] at (10.55,3.77) {No feedback};
\node[text=dfTeal] at (10.55,1.03) {Verifier};
\node[text=dfGray,font=\fontsize{8}{9.5}\selectfont] at (14.65,4.53) {reference};
\draw[dfTeal,line width=1.1pt] (15.12,3.45) -- (15.57,3.45) -- (15.57,1.7) -- (15.12,1.7);
\node[fill=white,text=dfTeal,align=center,font=\bfseries\fontsize{8}{9}\selectfont] at (15.57,2.575) {Primary\\contrast};
\node[anchor=west,text=dfGray,font=\fontsize{8}{9.5}\selectfont] at (0,.15)
 {Planned studies. One frozen corpus; prespecified evaluation panels and separate generalization conditions.};
\end{tikzpicture}
\caption{Two experiments isolate different contributions of verification.
\textbf{A} Models learn from matched data with verified or verifier-rejected
revision completions; an L1/L2-only condition provides a secondary control.
Both L3 conditions also contain the shared L1/L2 examples.
\textbf{B} A model revises its first proposal with or without feedback from
checking public evidence. Both revision arms receive the same call and output
limits; the one-call direct arm is descriptive. Brackets identify the two
primary comparisons. Networks denote model calls or trained model copies;
the circular arrow marks a second call that revises the proposal. The direct
arm bypasses this step. Sheets denote data or held-out tests.}
\label{fig:methods}
\end{figure*}

A verifier can influence learning in two distinct ways. It can determine
which examples a model learns from, or it can help the model revise an
unsuccessful proposal. These interventions need separate comparisons.
A gain from verified training would not establish that feedback helps,
and a gain from feedback would not establish the value of the training
data. Figure~\ref{fig:methods} shows the two prospective studies on a
common contextual population.

\subsection{Establishing a problem to learn}
The released evaluation makes baseline performance a substantive part of
the design. Before training, we will evaluate the pinned
Qwen2.5-7B-Instruct and Qwen2.5-32B-Instruct models
\citep{qwen2024technical} under both prompt styles. The training comparison
requires at least five percentage points of semantic error across at least
ten validation families. A failure to produce an admissible answer is
recorded separately from a validly expressed but unsuccessful revision.
If errors are predominantly failures of the response interface, that
interface must be diagnosed before attributing them to reasoning. If the
model already solves the tasks, the comparison closes. The complete
sampling and execution rules are in Appendix~\ref{app:agenda}.

\subsection{Study A: learning from verified revisions}
The central comparison holds the revision prompts and example draws
fixed while changing the semantic quality of their completions. One
condition contains verified revisions; the other contains syntactically
admissible revisions that the verifier rejects. Both contain the same
fact- and rule-abduction examples. Their difference tests whether correct
revision examples teach more than exposure to revision-shaped text.

A third condition contains only fact and rule abduction. Each revision
position is replaced by a lower-level example from the same input-length
bin, so the comparison does not simply add more training positions.
Optimizer settings, scheduled steps, token ceilings, and three seeds are
shared across conditions. Actual completion lengths may differ; realized
token exposure will therefore be reported. The three conditions correspond
to \texttt{l123}, \texttt{l123\_\allowbreak unverified}, and
\texttt{l12\_only} in the experimental implementation.

The primary outcome requires an admissible edit that resolves the anomaly,
transfers to the designated cases, and preserves all remaining predictions,
including the retention set. Verified versus rejected revisions is the
primary contrast; the comparison with lower-level training is secondary.
LoRA adaptation \citep{hu2022lora} uses the same settings and final scheduled
checkpoint for every condition. Training duration is selected by measured
runtime before outcomes, and retention on the released panel is evaluated
separately.

\subsection{Study B: revising with feedback}
A second attempt may help even when it receives no new semantic information.
To isolate feedback's contribution, the primary comparison gives both arms
an initial proposal and one revision. One receives the verifier's account
of that proposal's consequences; the other revises without it. Public
evidence, call and output limits, prompt style, and the rule for selecting
the final answer are identical. A one-call direct attempt provides a
lower-budget descriptive reference.

Feedback is restricted to the public anomaly and protected evidence.
It does not disclose private transfer or retention labels, reference rules,
or the set of valid answers. All three arms will be evaluated on the base
7B model and each of the three models trained on verified revisions.
Joint success is accompanied by paired wins and losses, collateral damage,
regressions, abstentions, actual tokens, and latency. These observations
will show whether feedback improves the final revision and what it costs,
without assuming that equal call limits produce equal computation.

\subsection{Interpreting the comparisons}
Source families are the units of uncertainty. Tasks from one family can
share both evidence and valid answers, so treating every response as
independent would overstate precision. Each co-primary contrast receives
a 97.5\% family-clustered interval, with paired training-seed effects
averaged within family before resampling. A positive result requires the
interval's lower bound to exceed zero; a practically meaningful gain also
requires a point estimate of at least five percentage points.

Family-disjoint and taxon-disjoint effects will be reported separately,
alongside retrieval, symbolic search, random admissible edits, and the
oracle. The model, prompt style, panel, and complete task/seed roster must
be fixed before outcomes. Descriptive grouping by semantic-equivalence
structure will assess sensitivity to shared answers. Null and adverse
results remain informative about which use of verification helps, or
fails to help, under the stated conditions. Hosted reference comparisons
follow instrument and runtime qualification; they support a claim about
contemporary leading models only if that cohort is separately verified.

\section{Discussion and outlook}
\label{sec:discussion}
A useful revision must change the right predictions. DeFAb makes
this requirement inspectable by representing an answer as an
edit to a finite theory and computing its consequences. The
released evaluation also shows that checking consequences at
the observed case can leave an important question unanswered.
An entity-local repair can pass without identifying which
other cases should share the exception.

The contextual contract addresses this limitation by requiring
both transfer and preservation. An exception must extend to
designated cases while leaving the remaining predictions
intact. A public-evidence symbolic solver is required to solve
every admitted contextual task, providing an explicit procedure
against which learned revisions can be assessed. The planned
studies will test whether verified examples help models acquire
that behavior and whether consequence feedback helps them
recover from an unsuccessful proposal. These are distinct
effects, and neither has yet been measured.

The current evidence supports a narrower contribution. The
released dataset and verifier make theory edits reproducible,
expose alternative valid repairs, and identify collateral
losses. The evaluation reveals a shortcut that motivates a
more discriminating task. Its GO-only source, small revision
population, and lack of an additional independent domain audit
limit the breadth of the conclusions.

Scientific reasoning also requires deciding whether a theory
describes the world. Its defaults, withheld exceptions, and anomalies are constructed
from curated assertions. Determining whether a proposed revision
describes a biological mechanism therefore requires further evidence.
Empirical discovery will require uncertain predictions and
external tests. The contribution here is a controlled way to
study one part of that process, revising an explanation while
accounting for the consequences of the change.

\label{maintextend}

\section*{Impact Statement}
This work studies verifiable theory revision as one component of scientific
reasoning. Its intended benefit is more inspectable evaluation and
training of models that propose explanations. A formally accepted edit
may still misrepresent biology because the finite theory and constructed
default are abstractions. Dataset scores are not evidence of clinical
validity or autonomous grounds for scientific intervention. Provenance,
explicit acceptance contracts, versioned releases, and transparent
negative results help users assess these limits. No new biological
discovery or deployment decision is claimed.

\bibliography{review_references}
\bibliographystyle{icml2026}

\newpage
\appendix
\section{Task semantics and an illustrative revision}
\label{app:semantics}
Consider facts $\mathit{member}(e_1,T)$,
$\mathit{member}(e_2,T)$, and $\mathit{member}(e_3,U)$, together with
$\mathit{parent}(e_i)$ for each entity. Let a default
$d:\mathit{parent}(x)\Rightarrow R(x,X)$ apply throughout the context.
Withhold the exception
$h_T:\mathit{member}(x,T)\leadsto\neg R(x,X)$ and its required priority
configuration. The resulting theory predicts $R(e_i,X)$ at all three
entities. Set $q=R(e_1,X)$,
$\mathit{Tr}=\{R(e_2,X)\}$, and
$\mathit{Kr}=\{R(e_3,X)\}$. The complete protected set includes every other
original expectation outside $\{q\}\cup\mathit{Tr}$.

An exception mentioning $e_1$ alone removes $q$ but leaves transfer
unchanged. A defeater applicable to every parent removes all three
predictions and fails retention. The class exception removes exactly the
two designated predictions. An alternative body is also accepted if
admissible and semantically successful. This is an illustration, not an
extracted source task.

\paragraph{Negation and absence.}
The evaluator does not equate absence of a proof with proof of a
complement. An L3 observation challenges a prediction rather than being
inserted into the theory as a premise. A defeater can block a prediction
without proving its opposite. These distinctions prevent a resolution
score from being interpreted as stronger entailment than required.

\paragraph{Finite verification.}
Grounding uses only declared record constants and permitted variable
positions. Closure is recomputed under the same proof theory before and
after the edit. Admissibility precedes semantic scoring; the raw answer,
resolution, and protected losses are retained. Exhaustive enumeration is
an oracle for the bounded language only, not for arbitrary first-order
theories or unrestricted scientific hypotheses.

\section{Released evaluation and reproducibility}
\label{app:released}
\begin{table*}[t]
\centering\small
\setlength{\tabcolsep}{5pt}
% Generated by evidence.py from registered, replayed evidence; do not edit.
\begin{tabular}{lrrrrrr}
\toprule
Pinned deployment & Direct success & Structured success & Contract & Capped & Semantic & Lost protected\\
 & /336 & /336 & failures & outputs & misses & predictions\\
\midrule
GPT-5.6 Sol & 336 & 336 & 0 & 0 & 0 & 0 \\
Claude Sonnet 5 & 306 & 244 & 111 & 0 & 11 & 0 \\
GPT-6.1 Sol & 336 & 336 & 0 & 0 & 0 & 0 \\
Claude Sonnet 5.5 & 133 & 336 & 203 & 0 & 0 & 0 \\
DeepSeek-V4-Pro & 273 & 38 & 336 & 0 & 25 & 0 \\
Kimi-K2.6 & 205 & 99 & 2 & 366 & 0 & 0 \\
\bottomrule
\end{tabular}

\caption{Audited 1.0.0 outcomes. Each row has 672 responses across two styles;
the three failure columns partition its non-successes. No collateral loss
was recorded among evaluated admissible L3 edits. Names identify the pinned runs and do
not imply a cross-provider ranking under every inference configuration.}
\label{tab:released}
\end{table*}
The release has fixed train, validation, and test partitions of 180,405,
22,550, and 22,552 records. Its immutable loading tag is
\texttt{canonical-v1.0.0}. The panel uses 150 fact-completion, 150 rule-abduction, and 36
conservative-revision tasks. Grouping is preserved, with at most three L3
tasks per family. Direct and structured prompts contain the same logical
task and differ in presentation. Parser outcomes, output caps,
admissibility, resolution, and preservation are separate fields.
L3 graded results summarized here use the current well-formedness-fixed
scorer.

The pinned deployment versions are GPT-5.6 Sol (2026-07-09), GPT-6.1 Sol
(2026-09-29), Claude Sonnet 5 and 5.5 (both version 2), DeepSeek-V4-Pro
(2026-04-23), and Kimi-K2.6 (2026-04-20). The GPT deployments use medium
reasoning effort; the Claude deployments use adaptive thinking with medium
effort. The other two use provider defaults. Each cell is one sample;
temperature and random seed were not standardized across providers. Thus
style, response interfaces, hidden reasoning, and output ceilings remain
conditions of these observations, not controlled explanations for every
cross-deployment difference.

There are 4,032 responses: six deployments times 336 tasks times two styles.
Of these, 2,978 are joint successes and 1,054 failures. The latter comprise
652 response-contract violations, 366 caps, and 36 admissible semantic
misses. No collateral loss was recorded among evaluated admissible L3
edits; L1/L2 tasks and invalid answers do not supply preservation trials. These fixed-panel counts
are descriptive; repeated styles do not create independent families.

Generic repair and bounded search both solve every panel task under the
same acceptance predicate. The release-wide role census identifies
3,144,222 rule bodies, all a single entity fact. The withdrawal probe
targets all 36 L3 tasks and causes protected losses on 34; those 34 are
rejected. Model saturation and checker probes answer different questions
and are not pooled into a learned-performance claim.

\section{Prospective experimental details}
\label{app:agenda}
\paragraph{Support and symbolic baselines.}
For a child term $X$ and parent $P$, let $A_X$ be eligible allowed-side
entities annotated to $X$ or a descendant, and $A_P$ those annotated to
$P$ or a descendant but not to $X$ or below. Support is
$|A_X|/(|A_X|+|A_P|)$. It is recorded at source census, realized context,
and split-pool stages; the stricter sub-panel requires at least $0.5$ at
each stage. The ordinary threshold remains $0.25$.
The first heuristic (H1) chooses the widest applicable membership class;
the second (H2) chooses the narrowest applicable non-membership class,
in each case consistent with the observed and protected public evidence.
H3 infers required sides from public evidence and searches for a separating
body. Complete success by H3 and the hidden-label oracle is a binding
instrument requirement, not a completed contextual result.

\paragraph{Frozen inputs and grouping.}
Contexts contain at most 12 entities, 6 species, 40 terms, and 4 ontology
hops. Each of the two prompt styles must fit 16,384 bytes; overlong prompts
are excluded before evaluation, never truncated. The family-disjoint panel
target is 300 L3, 150 L1, and 150 L2 tasks. The separate taxon challenge
targets 60 L3 tasks across at least 20 families. These are intended yields,
not counts from an accepted contextual corpus.

One contextual training population supports both named evaluation
conditions. Source taxon classes are reserved before regeneration;
families, shared constraints, and context entities remain grouped.
Train-only retrieval cannot index validation or test records. Leakage
audits check prompts, candidate order, family/entity identity, source
constraints, and overlap among all valid-answer classes.
Reference-answer novelty alone is insufficient. Original development
populations and failed registrations remain historical evidence.

\paragraph{Admission and stopping.}
Baseline models are Qwen2.5-7B-Instruct and Qwen2.5-32B-Instruct. A fixed
stratified selection uses 32 validation families with up to five tasks
per family, both prompt styles, a 4,096-token output ceiling, and strict
scoring without truncation. Batch sizes 1, 4, and 16 are compared on
throughput and output agreement, without choosing from task scores.
Semantic headroom must be at least five percentage points with errors
in at least ten families; response-interface failures remain separate.
Frozen panels are evaluated once per registered condition.

Binding instrument checks precede model evaluation. The 32-family
validation sample and ten-error-family training requirement remain even
if source yield is inadequate. Feasibility failure closes the affected
claim rather than silently weakening its criterion. A bounded split
amendment is prospective, with a fixed candidate and the original retained.
The taxon challenge's 60-task/20-family target is separate from standard
family-disjoint yield; a shortfall withholds the stronger transfer claim.

\paragraph{Training controls.}
The three conditions use seeds 42, 43, and 44. LoRA rank is 64, alpha is
128, learning rate is $2\times10^{-5}$, effective batch is 16, and maximum
sequence length is 8,192. All nine condition/seed runs use the same length.
If even 50 steps cannot fit the training and evaluation allocation,
training stops for a resource decision. No seed or control is dropped.

Unverified completions are syntactically admissible but semantically
rejected. Verified and rejected L3 completions share prompts and draws;
the lower-level control replaces each L3 position with an L1/L2 example
in the same input-length bin. Token ceilings, optimizer settings, and
scheduled steps are shared. Realized completion-token exposure is
reported rather than assumed identical. Twenty-step pilots select the
largest common duration among 50, 100, and 200 steps from runtime with a
25\% margin. Final scheduled checkpoints are evaluated; test selection,
seed reduction, and control deletion cannot rescue a comparison.

\paragraph{Feedback accounting.}
A fixed stratified hash selects 64 test families and two tasks per family
before outputs. If fewer families are available, the reduced scope is
reported without adaptive replacement. For each of 128 tasks, the direct
arm uses one call and each revision arm
two. Across the base model and three verified-SFT seeds, the ceiling is
$128\times(1+2+2)\times4=2{,}560$ calls. Revision arms share public
evidence and final selection, using the direct prompt style and a
4,096-token ceiling per call. Joint success includes admissibility,
resolution, preservation, transfer, and retention. Private transfer labels and valid-answer
sets never enter feedback. Malformed answers, refusals, abstentions,
caps, and regressions are retained rather than repaired using test outcomes.

\paragraph{Analysis.}
Co-primary contrasts receive separate 97.5\% family-clustered intervals.
Within each registered comparison, task effects are averaged within
family and seed, then across paired seeds before resampling families.
Every family receives equal weight. The complete task/seed roster, model,
prompt style, and panel are sealed before outcomes. No cross-panel or
outcome-selected pooling is permitted.
Descriptive semantic-equivalence grouping assesses sensitivity to repeated
valid-answer structure. The interval must lie above zero for a positive
claim; practical importance also requires a point estimate of at least
five percentage points. Smaller, null, adverse, and inconclusive findings
remain distinguishable.

\paragraph{Resources.}
The 70 single-H200-hour ceiling comprises 2 hours of qualification,
8 of baseline evaluation, 36 of matched training and evaluation, 12 of
feedback, and 12 of unreleased contingency. Contingency requires a
diagnosed registered need. CPU preparation precedes GPU allocation
wherever practical. Backend/adapter agreement, actual compute progress,
throughput, input identity, and verified custody are execution gates.
Exit status or allocation alone is not qualification. Hosted reference
work has a complete-comparison cap of \$150 inside the existing \$500
envelope. These are limits, not performance results or spending targets.

\paragraph{Reporting.}
The contextual release remains separate from 1.0.0. Corpus size, support
bands, valid-answer diversity, attrition, heuristic performance, model
headroom, and failed gates accompany scores. A contextual success supports
only the declared generalization condition.

\ifdefined\DeFAbNamed
  \section{Public release and manuscript status}
\label{app:public-release}
The released dataset is version 1.0.0 at
\url{https://huggingface.co/datasets/PatrickAllenCooper/DeFAb}, pinned by
\texttt{canonical-v1.0.0}. Its verified revision is
\texttt{953208e3429552c7aec8\allowbreak bad57df6b62a4be70772}.
Contextual \texttt{defab/2} work is under development, not a published 2.0.0
release. First-release records and scores are not pooled with development
corpora.

The released verifier and installation instructions are available in the
dataset's \href{https://huggingface.co/datasets/PatrickAllenCooper/DeFAb/tree/main/research/2026-10-22/code}{code directory}.
The accompanying source package includes a checksum and installation guide.
The frozen code snapshot supplies the \texttt{blanc} imports used by the
dataset card; its schema, scorer, theory representation, and finite proof
kernel match the released implementation byte for byte. The snapshot's
\texttt{v1.1} filename is a manuscript code label, not a dataset version.
Neither trained adapters nor fine-tuning results are included. Training
scripts and the public SFT data view do not constitute completed training
experiments; the controlled learning studies remain prospective.

This is a research-direction preprint, not an accepted ICML paper. It uses
the verified 2026 style kit pending a verified 2027 kit. Venue requirements
must be rechecked before submission. The anonymous version is a working
copy; the named draft is for public circulation.

\fi

\end{document}